%% file: main.tex
\relax
\documentclass[letterpaper]{article} 
\usepackage{aaai17simple}  
\usepackage{times}  
\usepackage{helvet}  
\usepackage{courier}  
\usepackage{url}  
\usepackage{graphicx}  

\usepackage{fullpage}
\usepackage{booktabs}
\usepackage{multirow}
\usepackage{subfig}
\usepackage[noend]{algorithmic}
\usepackage[plain]{algorithm}

\usepackage{amsmath, amssymb, amsfonts, amsthm}

\newtheorem{theorem}{Theorem}[section]
\newtheorem{lemma}[theorem]{Lemma}
\newtheorem{corollary}[theorem]{Corollary}

\newtheorem{definition}{Definition}

\begin{document}
%
\title{A Formal Characterization of the Local Search Topology \\of the Gap Heuristic}
\author{Richard Valenzano\\
        Department of Computer Science\\
        University of Toronto\\Toronto, Ontario, Canada\\
        rvalenzano@cs.toronto.edu \and
		Danniel Sihui Yang\\
		Department of Computer Science\\
		University of Toronto\\
		Toronto, Ontario, Canada\\
		dannielyang1996@gmail.com}
\date{}

\maketitle
\begin{abstract}
The pancake puzzle is a classic optimization problem that has become a standard benchmark for heuristic search algorithms.
In this paper, we provide full proofs regarding the local search topology of the gap heuristic for the pancake puzzle.
First, we show that in any non-goal state in which there is no move that will decrease the number of gaps, there is a move that will keep the number of gaps constant.
We then classify any state in which the number of gaps cannot be decreased in a single action into two groups: those requiring $2$ actions to decrease the number of gaps, and those which require $3$ actions to decrease the number of gaps.
\end{abstract}

\section{Background}
\input{background}

\section{Locked States and Gap Neutral Moves}
\input{neutral_locked}

\section{Topology of the Gap Heuristic}\label{sec:topology}
\input{topology}

\section{Conclusion}
In this work, we have provided the complete proofs underlying a characterization of the topology of the gap heuristic.
First, we showed that there is a gap neutral move in any non-goal locked state.
Then, we provided a classification that organizes states into whether the number of gaps can be decreased in $1$, $2$, or $3$ actions.

\bibliography{refs}
\bibliographystyle{aaai}

\end{document}

%% file: background.tex
In this section, we provide background on the pancake problem and define the notation used in the rest of the paper.

\subsection{Sequences and Permutations}

In this paper, we represent a \textbf{sequence} $\sigma$ of $k$ elements from some set as $\sigma = \langle e_1, ..., e_k \rangle$, where $\sigma[i]$ then refers to the $i$-th element of $\sigma$ (\textit{i.e.} $\sigma[i] = e_i$)\footnote{As is convention in the pancake puzzle literature, the first element of the permutation is at location 1.}. 
If $\sigma' = \langle g_1, ..., g_{k'} \rangle$, we use $\sigma \circ \sigma'$ to denote the \textbf{concatenation} of these sequences: $\langle e_1, ..., e_k, g_1, ..., g_{k'}\rangle$.
A \textbf{permutation} $\pi$ of size $N$ is a sequence of the natural numbers from $1$ to $N$, such that each element in the sequence is unique.

\subsection{The Pancake Puzzle Problem}

An \textbf{$N$-pancake puzzle state} is a stack of $N$ different sized pancakes.
We represent this stack with a permutation of size $N$, where entry $i$ refers to the $i$-th smallest pancake and the order of the numbers in the permutation corresponds to the order of the pancakes in the stack from top to bottom.
For example, $\langle 2, 1, 4, 3\rangle$ represents a $4$-pancake stack in which the second smallest pancake is at the top of the stack.

In any $N$-pancake state, there are $N-1$ applicable \textbf{actions} or \textbf{moves}, denoted by $M_2$, $M_3$, ..., $M_N$.
Action $M_k$, called a \textbf{$k$-flip}, reverses the order of the first $k$ values in the stack.
Where $M_k(\pi)$ denotes the permutation that is the result of applying action $M_k$ to $\pi$, this means that $M_k(\pi)[1] = \pi[k]$, $M_k(\pi)[2] = \pi[k-1]$, and so on. 
For example, $M_3(\langle 2, 1, 4, 3 \rangle) = \langle 4, 1, 2, 3 \rangle$.
\begin{definition}
Given $N$-pancake state $\pi_\mathrm{init}$, the \textbf{$N$-pancake puzzle task} is to find the shortest or \textbf{optimal} sequence of flips that transforms $\pi_\mathrm{init}$ into state $\pi_\mathrm{goal} = \langle 1, 2, ..., N \rangle$.
\end{definition}

\subsection{The Gap Heuristic}

A \textbf{heuristic function} $h$ is a function from the set of states to the set of non-negative real numbers, where $h(\pi)$ is referred to as the heuristic value of $\pi$.
A heuristic $h$ is said to be \textbf{admissible} if for every state $\pi$, $h(\pi) \leq h^*(\pi)$.

The \textbf{gap} heuristic \cite{Helmert:gaps}, which we denote by $h^G$, is most easily formally defined using the \textbf{extended} permutation $\pi^e$ of $\pi$.
$\pi^e$ is defined as $\pi \circ \langle N + 1\rangle$.
The value $N+1$ can be thought of as the plate below the pancake stack, though we often refer to it as the $N+1$-st pancake.
Moreover, due to the one-to-one correspondence between $\pi$ and $\pi^e$ we often refer to $\pi[N + 1]$, the ``$N+1$-st pancake" of $\pi$, or ``location $N+1$" in $\pi$.

For any $j$ where $1 \leq j\leq N$, an \textbf{adjacency} is said to occur in $\pi^e$ between locations $j$ and $j+1$, or between pancakes $\pi^e[j]$ and $\pi^e[j+1]$, if $|\pi^e[j] - \pi^e[j+1]| = 1$.
A \textbf{gap} is said to occur between those locations (or those pancakes) if an adjacency does not occur.
$h^G(\pi)$, is then given by the count of the number of gaps in $\pi^e$:
\begin{align*}
    h^G(\pi) = |\lbrace j \mid 1 \leq j \leq N, |\pi^e[j] - \pi^e[j+1]| > 1 \rbrace|
\end{align*}
Since any action can only add or remove at most one gap and there are no gaps in $\pi_{\mathrm{goal}}$, $h^G$ is admissible.

If action $M_i$ removes a gap when applied to state $\pi$ (\textit{i.e.} $h^G(M_i(\pi)) = h^G(\pi) - 1$), then $M_i$ is called a \textbf{gap decreasing} move in $\pi$.
Similarly, $M_i$ is a \textbf{gap increasing} move if it introduces a gap, while if it replaces one gap with another or one adjacency with another, $M_i$ is a \textbf{gap neutral} move.

We observe that there are always at most two gap decreasing moves in any state.
This is because $M_i$ can only resolve a gap (if one exists) between locations $i$ and $i+1$, if $\pi[1]$ is adjacent to $\pi[i+1]$ in $\pi_{\mathrm{goal}}$, and this is only true if $\pi[i+1] = \pi[1] + 1$ or $\pi[i+1] = \pi[1] - 1$.
However, in many states there are no gap decreasing moves.
These states are said to be \textbf{locked}.

%% file: neutral_locked.tex
In this section, we formally prove that there is a gap neutral move in every non-goal locked state.
We begin by showing that there is at least one gap in any non-goal state.

\begin{lemma}\label{lemma:gaps_and_goal_state}
$\pi$ is a goal state if and only if $h^G(\pi) = 0$.
\end{lemma}
\begin{proof}
The fact that if $\pi$ is a goal state, then $h^G(\pi) = 0$ is obvious.
As such, let us assume that $h^G(\pi) = 0$
Then there is no gap between locations $N$ and $N+1$ of $\pi$, which ensures that $\pi[N] = N$.
Similarly, there is no gap between locations $N-1$ and $N$ of $\pi$, which ensures that $\pi[N-1] = N-1$.
This argument can be extended to show that for any $1 \leq i \leq N$, $\pi[i] = i$.
As such, $\pi$ is the goal state.
\end{proof}

We will now use this lemma to show that there is always a gap neutral move in any non-goal locked state.
\begin{theorem}
If $\pi$ is a locked state that is not the goal, then there is a gap neutral move in $\pi$.
\end{theorem}
\begin{proof}
Let $\pi$ be a locked state that is not the goal.
First, we note that $\pi[1] \neq N$.
This is because if $\pi[N] = N$, then $\pi[N] \neq N$, and so there is a gap between locations $N$ and $N+1$.
Thus, $M_N$ is a gap decreasing move which contradicts the assumption that $\pi$ is locked.
There are now two cases to consider.

\vspace{+0.08in}
\noindent
\textit{Case 1:} $\pi[1] > 1$.

Since $\pi[1] > 1$, $\pi[1] < N$ since otherwise $\pi$ is not locked by the argument above.
As such, let $e$ and $e'$ be the two pancakes that should be beside $\pi[1]$ in the goal state (\textit{i.e.} $|\pi[1] - e| = |\pi[1] - e'| = 1$).
If $\pi[\ell] = e$ and $\pi[\ell'] = e'$, we can assume that $\ell > \ell'$ without loss of generality.
Since $1 < \ell' < \ell$, $\ell > 2$.

Now because $\pi$ is locked, there is an adjacency between locations $\ell - 1$ and $\ell$.
Since $\ell > 2$, $M_{\ell - 1}$ is a valid action.
Moreover, $M_{\ell - 1}$ will move $\pi[1]$ on top of $e$, thereby replacing one adjacency with another.
As such, $M_{\ell -1}$ is a gap neutral move, thus guaranteeing that one such gap neutral move exists in $\pi$ in this case.

\vspace{+0.08in}
\noindent
\textit{Case 2:} $\pi[1] = 1$.

There are now two subcases to consider.
First, $\pi[2] \neq 2$.
This means that $\pi[\ell] = 2$ for some $\ell > 2$ and so $M_{\ell - 1}$ is a valid action.
This action will simply replace one adjacency with another by the same argument as in Case 1, and so the statement holds in this case.

Now suppose that $\pi[2] = 2$.
Since $\pi$ is not a goal state, there must exist a gap between some two locations $\ell$ and $\ell+1$ by Theorem \ref{lemma:gaps_and_goal_state}.
Because there $\pi[1] = 1$ and $\pi[2]=2$, there is no gap between locations $1$ and $2$ in $\pi$, and so $\ell \geq 2$.
Thus, $M_\ell$ is a valid move.
Since $\pi[2] = 2$, this means that $\pi[\ell+1] \neq 2$ and so $M_{\ell}(\pi)$ will also have a gap between locations $\ell$ and $\ell + 1$.
Thus, $M_\ell$ is a gap neutral move in $\pi$ which replaces one gap with another.
As such, there is a gap neutral move in $\pi$ in this case.

\vspace{+0.08in}
Having handled all cases, the statement holds.
\end{proof}

%% file: topology.tex
In this section, we extend the work of \citeauthor{Fischer:two_approx} \shortcite{Fischer:two_approx} and provide a classification of states according to the size of the plateaus around them.
To simplify this analysis, we assume that in all states, there is gap between locations $N$ and $N+1$. 
Doing so removes the postfix of a state if it is already sorted, since this portion of the state will have no impact on the number of gaps or the optimal solution cost.
For example, where $\pi = \langle 2, 1, 4, 3 \rangle$ and $\pi' = \langle 2, 1, 4, 3, 5, 6, 7 \rangle$, clearly $h^G(\pi) = h^G(\pi')$ and $h^*(\pi) = h^*(\pi')$.

We begin with some additional notation.
Following \citeauthor{Hoffmann:topology} \shortcite{Hoffmann:topology}, a \textbf{plateau} for $h$ is a connected set of one or more states that all have the same heuristic value.
An \textbf{exit} from a plateau with heuristic value $\ell$ is a state $\pi$ such that $h(\pi) = \ell$ and there is some neighbour $\pi'$ of $\pi$ such that $h(\pi') < h(\pi)$.
The \textbf{exit distance} of $h$ from a state $\pi$ is the minimum number of actions needed to reach an exit.
Note that this means than any exit has an exit distance of $0$.

We also say that consecutive locations $i, i+1, ..., i + j$ in a permutation $\pi$ is a \textbf{strip} of size $j + 1$ if there are no gaps between the pancakes in those locations, and that sequence of locations is maximal (\textit{i.e.} on either side of the strip there is a gap or the end of the permutation).
A strip of size $2$ or more is \textbf{descending} if $\pi[i] > \pi[i+1] > ... > \pi[i+j]$, and \textbf{ascending} otherwise.
Two strips from $i$ to $i+j$ and $i'$ to $i' + j'$ where $i \leq i+j < i' \leq i'+j'$ are \textbf{in order} if the pancakes in the  strip from $i$ to $j$ are smaller than the pancakes in the strip from $i'$ to $j'$.
The \textbf{first strip} is the one starting at location $1$, and the \textbf{rightmost} strip is the one ending at location $N$.
Where $\ell$ is the size of the first strip of a state $\pi$ that has at least two strips, the \textbf{second strip} starts at location $\ell + 1$. 
The remaining strips are named similarly.
For example, $\langle 1, 2, 3, 5, 4 \rangle$ has two strips: the first strip is an ascending strip of size $3$ from locations $1$ to $3$, and a descending strip of size $2$ from location $4$ to $5$.
The latter strip is the second or rightmost one, and the two strips are in order.

We now define the following family of states:
\begin{definition}
$\pi$ is a \textbf{Fischer-Ginzinger (FG)} state if and only if $\pi$ has at least two strips, and all strips in $\pi$ are descending, have a size of at least two, and are in order.
\end{definition}
For example, $\langle 3, 2, 1, 5, 4 \rangle$ is an FG state, while $\langle 1, 2, 4, 3 \rangle$ and $\langle 2, 1, 3, 5, 4 \rangle$ are not FG states since they have an ascending strip and strip of size 1, respectively.

We can now characterize states according to their exit distance.
First, we notice that any state in which there is a gap decreasing move has an exit distance of $0$ by definition.
For locked states, consider the following corollary of Lemma 5 from Fischer and Ginzinger \shortcite{Fischer:two_approx}:
\begin{corollary}\label{cor:locked_exit_distance}
The exit distance of $h^G$ for any locked state that is not an FG state is 1.
\end{corollary}
Fischer and Ginzinger proved this by providing appropriate sequences of actions that could decrease the number of gaps for all possible cases of non-FG locked states.
Fischer and Ginzinger also provided a method for sorting any FG state $\pi$ using at most $2\cdot h^G(\pi)$ actions.
This method always decreases the number of gaps in $\pi$ in $3$ actions, thus guaranteeing that the exit distance of any FG state is at most $2$.
However, this method does not show that this is always necessary, and thus does not provide a characterization of FG states according to their exit distance.

To provide such a characterization, we define an \textbf{easy FG state} as an FG state with exactly $2$ strips such that the rightmost strip has a size of $2$.
We can now show the following:

\begin{theorem}\label{thm:easy_fg_exit_distance}
If $\pi$ is an easy FG state, then the exit distance of $h^G$ for $\pi$ is $1$ and $h^*(\pi)=3$.
\end{theorem}
\begin{proof}
Let $\pi$ be an easy FG state.
Since all FG states are locked, the exit distance of $\pi$ is at least $1$.
$\pi$ will also necessarily have the following form $\langle N-2, ..., 1, N, N-1 \rangle$.
As such, $\pi$ has two gaps, one of which can be removed by applying $M_{N-1}$ and then $M_N$ to reach state $\pi' = \langle N-1, N-2,..., 1, N \rangle$.
Applying $M_{N-1}$ and then $M_N$ will lead to state $\pi' = \langle N-1, N-2,..., 1, N \rangle$.
Since $\pi'$ has a single gap, the exit distance of $\pi$ is $1$.

Applying $M_{N-1}$ to $\pi'$ completes the sort of $\pi$, for a solution to $\pi$ with cost $3$.
Since the fact that $\pi$ is locked means that $h^*(\pi) \geq h^G(\pi) + 1 \geq 3$, this solution is optimal.
\end{proof}

Let us now contain the remaining FG states, which we call \textbf{hard FG states}.
We now show the following:
\begin{theorem}\label{thm:hard_fg_exit_distance}
The exit distance of $h^G$ is $2$ for any hard FG state. 
\end{theorem}
\begin{proof}
Let $\pi$ be a hard FG state, and let $\ell \geq 2$ be the size of the rightmost strip of $\pi$.
Then $\pi$ is of the form $\langle e_1, ..., e_{N-\ell}, N, N-1, ..., N -\ell + 1 \rangle$, where $e_1$ to $e_{N-\ell}$ are pancakes in the range from $1$ to $N-\ell$.
Consider following sequence of moves: $M_N$, $M_\ell$, and then $M_N$.
The result of $M_N$ will be $\langle N -\ell + 1, ..., N-1, N, e_{N-\ell}, ..., e_1 \rangle$.
$M_{\ell}(M_N(\pi))$ will then be $\langle N, N-1, ..., N-\ell _ 1, e_{N-\ell}, ..., e_1 \rangle$.
The final application of $M_N$ will lead to the state $\langle e_1, ..., e_{N-\ell}, N -\ell + 1, ..., N-1, N \rangle$.
The gap between locations $N$ and $N+1$ in $\pi$ is now gone, while all other gaps remain.
Thus, $h^G(M_N(M_{N-1}(M_N(\pi)))) = h^G(\pi) - 1$, and so the exit distance of $\pi$ is at most $2$.

Let us now show that the exit distance of $\pi$ is greater than $1$.
Since $\pi$ is locked, the exit distance is at least $1$.
We will now show that for any $M_i$, $M_i(\pi)$ cannot be an exit.
There are three main cases to consider.

\vspace{+0.08in}
\noindent
\textit{Case 1:} $M_i$ is a gap increasing move.

Because of the consistency of $h^F$ and $h^G(M_i(\pi)) = h^G(\pi) + 1$, no neighbour of $M_i(\pi)$ can have fewer than $h^G(\pi)$ gaps.
As such, $M_i(\pi)$ cannot be an exit in this case.

\vspace{+0.08in}
\noindent
\textit{Case 2:} $M_i$ is a gap neutral move that replaces one gap with another.

Since the only gaps are between the strips, this means that location $i$ must be the end of some strip.
Suppose that it is at the end of the first strip (\textit{i.e.} $i$ is the length of the first strip).
In this case, $\pi$ is of the form $\langle i, i-1, ..., 1, ... \rangle$ and $M_i(\pi)$ will be $\langle 1, 2, ..., i ... \rangle$.
Pancake $1$ is already on top of $2$, and so $M_i(\pi)$ is locked and thus cannot be an exit.

Now suppose that the size of the first strip is smaller than $i$.
Let $\ell$ be the size of the strip ending at location $i$ and let $\ell'$ be the size of the strip right before that strip.
Then $\pi$ is of the following form:
\begin{align*}
    \langle ...,  i - \ell, i - \ell - 1, ..., i - \ell - \ell' + 1, i, ... ,  i- \ell + 2, i - \ell + 1...\rangle
\end{align*}
In this state, pancake $i - \ell + 1$ is in location $i$.
Then $M_i(\pi)$ will be as follows:
\begin{align*}
    \langle i - \ell + 1, i - \ell + 2, ...,  i, i - \ell - \ell' + 1, ..., i - \ell - 1, i - \ell, ... \rangle
\end{align*}
Sinze $i - \ell + 1$ is the top pancake, the only possible locations for gaps that can be removed in $M_i(\pi)$ are above pancakes $i - \ell + 2$ and $i - \ell$.
However, these pancakes do not have gaps above them.
As such, $M_i(\pi)$ is locked and is not an exit.
Therefore, $M_i(\pi)$ cannot be an exit if $M_i$ is a gap neutral move that replaces one gap with another.

\vspace{+0.08in}
\noindent
\textit{Case 3:} $M_i$ is a gap neutral move that replaces one adjacency with another. 

Let $\ell$ and $\ell'$ be the sizes of the first and second strips in $\pi$.
By the definition of an FG state, $\pi$ is of the following form:
\begin{align*}
    \langle \ell, \ell-1,..., 1, \ell + \ell', .., \ell + 2, \ell + 1, ... \rangle
\end{align*}
In this state, $\ell + 1$ is in location $\ell + \ell'$.
Since the top pancake $\ell$ is already adjacent to pancake $\ell-1$, the only gap decreasing move that replaces one adjacency with another is $M_{\ell + \ell' - 1}$ (\textit{i.e} $i = \ell +\ell' -1$).
Applying this action to $\pi$ results in the following:
\begin{align*}
    \langle \ell + 2, ... \ell+\ell', 1, ... \ell-1, \ell, \ell + 1, ... \rangle
\end{align*}
Because $\pi$ is a hard FG state, it is guaranteed to have at least $3$ strips or a second strip with a size at least $3$.
If the second strip has a size of at least $3$, then the second strip in $\pi$ ends in $\ell + 3, \ell + 2, \ell + 1$ and so $M_i(\pi)$ has the following form:
\begin{align*}
    \langle \ell + 2, \ell + 3, ... \ell+\ell', 1, ... \ell-1, \ell, \ell + 1, ... \rangle
\end{align*}
Since $\ell + 2$ is already beside $\ell + 3$ and $\ell$ is on top of $\ell+1$, there is no gap decreasing move in $M_i(\pi)$. 
As such, $M_i(\pi)$ is locked and so it cannot be an exit.

If the second strip of $\pi$ is of size $2$ (\textit{i.e.} $\ell' = 2$), then there are at least $3$ strips.
Let $\ell''$ be the size of the third strip.
$\pi$ will necessarily have the following form:
\begin{align*}
    \langle \ell, \ell - 1,..., 1, \ell + 2, \ell + 1, \ell + 2 + \ell'', ..., \ell + 4, \ell + 3, ... \rangle
\end{align*}
We can now see that $M_{\ell + \ell' - 1}(\pi)$ has the following form:
\begin{align*}
    \langle \ell + 2, 1, ..., \ell - 1, \ell, \ell + 1, \ell + 2 + \ell'', ..., \ell + 4, \ell + 3, ... \rangle
\end{align*}
$\ell + 1$ and $\ell + 3$ are both below pancakes that they should be adjacent to, so this state remains locked and thus cannot be an exit.
Thus, $M_i(\pi)$ is not an exit whenever $M_i$ is a gap neutral move that replaces one adjacency with another. 

\vspace{+0.08in}
Since $M_i(\pi)$ cannot be an exit in all cases, the exit distance of $\pi$ is at least $2$.
This completes the proof.
\end{proof}

The results above show that the exit distance of any pancake state is at most $2$.